\documentclass[journal]{IEEEtran}
\usepackage{algorithm} 
\usepackage{algorithmic} 
\usepackage{multirow} 
\usepackage{amsmath}
\usepackage{xcolor}
\usepackage{amsfonts}
\usepackage{bm}
\usepackage{array}
\usepackage{booktabs}  
\usepackage{threeparttable} 

\usepackage{hyperref}
\usepackage{amssymb}
\usepackage{multirow}
\usepackage{graphicx}
\usepackage{subfigure}
\usepackage{verbatim}
\usepackage{booktabs}
\usepackage{latexsym}
\usepackage{amsthm}

\ifCLASSINFOpdf
\else
\fi

\newtheorem{theorem}{Theorem}

\begin{document}

\title{Simplifying Graph Convolutional Networks with Redundancy-Free Neighbors}
\author{Jielong~Lu,
        Zhihao~Wu,
        Zhaoliang~Chen,
        Zhiling~Cai,
        Yueyang~Pi,
        Shiping~Wang}
\maketitle

\begin{abstract}
In recent years, Graph Convolutional Networks (GCNs) have gained popularity for their exceptional ability to process graph-structured data. 
Existing GCN-based approaches typically employ a shallow model architecture due to the over-smoothing phenomenon.
Current approaches to mitigating over-smoothing primarily involve adding supplementary components to GCN architectures, such as residual connections and random edge-dropping strategies. 
However, these improvements toward deep GCNs have achieved only limited success.
In this work, we analyze the intrinsic message passing mechanism of GCNs and identify a critical issue: messages originating from high-order neighbors must traverse through low-order neighbors to reach the target node.
This repeated reliance on low-order neighbors leads to redundant information aggregation, a phenomenon we term over-aggregation.
Our analysis demonstrates that over-aggregation not only introduces significant redundancy but also serves as the fundamental cause of over-smoothing in GCNs.
Motivated by this discovery, we introduce a novel framework named redundancy-free graph convolutional network,
where the neighbors of the graph are hierarchically organized so that the multi-order neighbor sets of a specific node do not intersect.
This organizational structure enables high-order neighbors to directly propagate their messages to the target node, thereby effectively avoiding duplicate aggregation.
The layer number of the proposed method adapts according to the graph structure, eliminating the need for manual adjustments to capture node information at specific distances.
The experimental results on sixteen real-world datasets demonstrate the superior performance of the proposed method on node- and graph-level tasks.

\begin{IEEEkeywords}
 Graph convolutional network, over-smoothing, hierarchical aggregation, over-aggregation.
\end{IEEEkeywords}

\end{abstract}

\section{Introduction}
\IEEEPARstart{G}{RAPHS} 
 are powerful and popular data structures for representing relational data and are widely utilized in various fields.
For example, graphs can be used to depict purchase relationships in recommendation systems \cite{ChenSWGLZ23, he2023dynamically} and molecular chemical bonding connections in biomedicine \cite{LeeHNK0P23}.
Graph Convolutional Networks (GCNs) \cite{GCN} have garnered significant attention for their excellent capability to process graph-structured data \cite{PIYUEYANG}. 
They have found applications in a variety of graph-related tasks, including motion capture \cite{wang2023dynamic,li2023graph, zhong2022spatio}, graph classification \cite{xie2021federated,yin2023coco,li2019semi} and traffic forecast \cite{jiang2023spatio, bai2020adaptive,guo2021hierarchical}.
However, most existing GCNs employ a 2-layer convolutional architecture, which limits their ability to capture long-distance information.
To overcome this limitation, recent endeavors have attempted to stack multiple graph convolutional layers.
\begin{figure}[!htbp]
	\centering
	\includegraphics[width=\linewidth]{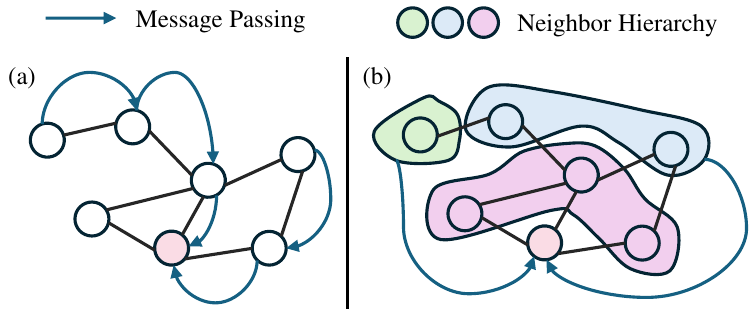}\\
	\caption{Comparison of message passing techniques on graphs. Subfigure (a) illustrates the stepwise information transfer from a node to a target node. Subfigure (b) depicts the proposed model that organizes the neighbors of a node into a hierarchical structure, reducing the risk of overutilizing information from intermediate nodes.}
	\label{useactivate}
\end{figure}
Nevertheless, this often leads to a significant performance decline due to the over-smoothing issue, where node representations become indistinguishably similar.
To address this issue, many researchers improve the performance of GCNs of deeper layers by utilizing random edge-dropping strategies on graphs \cite{DropEdge, SAEDGEDROP}, imposing regularization constraints \cite{PairNorm,guo2023contranorm}, and adding residual connections \cite{JKNET, APPNP, PDEGCN, GCNII, omegagcn}.
By adding supplementary components to typical GCNs, these methods have significantly improved the performance of deep GCNs.


Despite prior successes in alleviating the over-smoothing problem, deep GCNs constructed using the above methods still provide limited gains. 
For example, the performance achieved by stacking numerous layers using the aforementioned methods is only similar to a classical $2$-layer GCN, and in some cases, they may be even worse.
Besides, only $L$ hops are commonly required to capture comprehensive information, where $L$ is the diameter of the graph. 
For denser graphs, fewer hops are enough, e.g., each node can access the information of all other nodes within one hop in a complete graph.
However, existing deep GCNs struggle to adaptively determine the proper depths for different graphs.
These problems raise the question of whether existing solutions touch the essence behind the over-smoothing issue and deeper GCNs.

In search of the answer, we delve into the message passing mechanism of GCNs from the perspective of neighbor hierarchy.
In the existing GCNs, when a higher-order neighbor needs to pass a message to the target node, the message must traverse through lower-order neighbors, as illustrated in Figure \ref{useactivate} (a).
This results in the repeated aggregation of information from lower-order neighbors, leading to increased computational load as well as redundant information passing, which we call the over-aggregation issue.
For further analyzing this issue, we provide a theoretical demonstration showing that deeper GCNs may excessively aggregate information from neighboring nodes.
As a result, this redundant neighbor information squeezes out the nodes' self-information, diminishing the distinctiveness of the nodes and so that causing the over-smoothing issue.
Taking the widely adopted residual connections enhancing GCNs as an example,
we find that the phenomenon of over-utilizing the lower-order neighbors persists, although this kind of approach tries to enhance the self-attention of the nodes in the deep layer.
Consequently, existing approaches to the over-smoothing problem may not identify and address the deeper factor, the over-aggregation issue, which contributes to their weaknesses.

Building upon these insights, we introduce a framework named Redundancy-Free Graph Convolutional Network (RF-GCN), which specifically designs a simple yet effective strategy for tackling the over-aggregation problem. 
The proposed framework first hierarchically groups neighbors of each node into several non-overlapping sets according to their hops.
The neighbor information from each set is then aggregated by shortest paths to ensure that messages from higher-order neighbors reach the target node directly without relying on lower-order neighbors.
As shown in Figure \ref{useactivate} (b), distinguishing from the traditional message passing mechanism, RF-GCN enables the target node to interact directly with different hierarchies of neighbor sets.
In this way, we eliminate duplication in aggregation so that the nodes' self-information will not be squeezed by too much neighbor node information.
And the maximum depth of the network is limited to the diameter $L$ of the graph, which is just enough to capture the information of the entire graph.
\begin{figure*}[!htbp]
	\centering
	\includegraphics[width=0.97\textwidth]{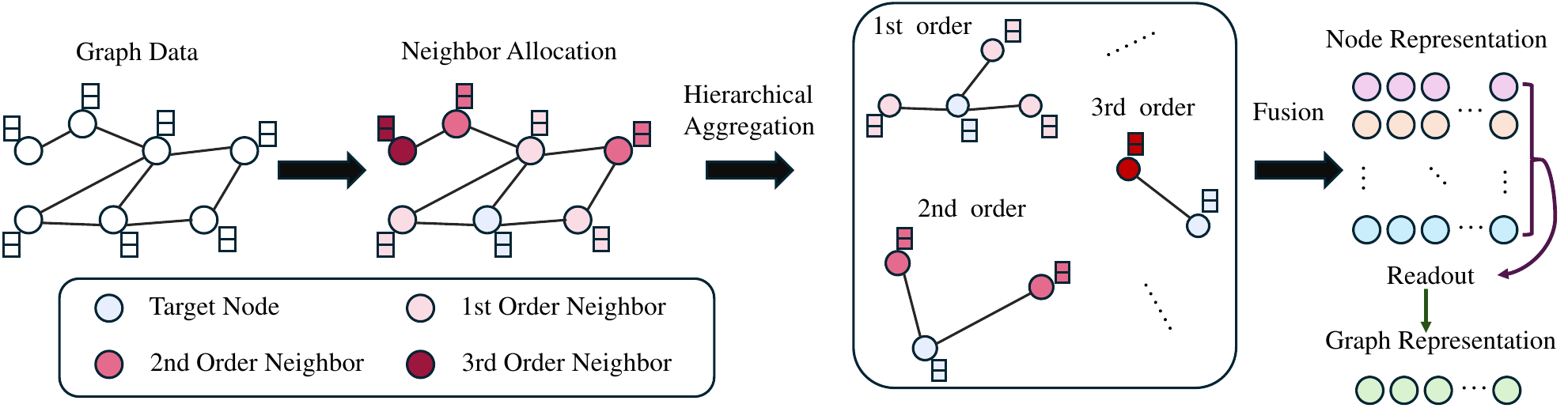}\\
	\caption{An overview of the proposed RF-GCN framework includes detailing the assignment of neighborhoods and hierarchical aggregation of target nodes, followed by the PPR fusion of the resulting representations.}
	\label{Framework}
\end{figure*}

The main contributions of our work are as follows:
\begin{enumerate}
\item  Theoretical analysis of GCNs reveals that over-utilizing low-order neighbors diminishes self-message attention, with persistent over-aggregation due to an unchanged aggregation mechanism despite increased attention by utilizing residuals.
\item Motivated by the previous analysis, we propose a deep GNN framework called RF-GCN, in which neighbors are hierarchically aggregated so that the information of each neighbor is utilized only once, effectively avoiding over-aggregation.
\item The proposed model demonstrates excellent performance in experiments conducted on 16 real-world datasets on node- and graph-classification tasks.
\end{enumerate}

The rest of the paper is organized as follows: Section \ref{SEC2} reviews recent advances in GCNs and explores strategies to mitigate the over-smoothing issues. Section \ref{SEC3} analyzes existing models and introduces the RF-GCN model. Section \ref{SEC4} assesses the performance of RF-GCN through extensive experimental evaluations across various settings. Finally, Section \ref{SEC5} provides a summary of our findings and conclusions.

\section{Related Work} \label{SEC2}
In this section, we begin with an introduction to graph convolutional networks and subsequently address the issue of over-smoothing.

\subsection{Graph Convolutional Networks}
GCNs are derived from the first-order truncation of ChebyNet \cite{Chebynet}.
As GCNs have proven effective in embedding graph-structured data and have gained widespread attention, numerous variants have emerged to address various scenarios.
SGC \cite{SGC} simplified GCNs by removing nonlinearities and integrating the weight matrix.
GAT \cite{GAT} was utilized for induction and transformation due to its ability to assign unique weights to nodes within the neighborhood without relying on the pairwise graph structure.
GIN \cite{GIN} made GNN as powerful as the Weisfeiler-Lehman graph isomorphism test by modeling injective functions.
LD2 \cite{LD2}  decoupled graph propagation from the generation of expressive embeddings to streamline the learning process, achieving optimal training time complexity with a memory footprint that remains constant regardless of graph size.
LGCN-SGIB \cite{zhong2023learnable} employed dual GCN-based meta-channels to explore both local and global relations and incorporated graph information bottleneck to minimize noisy data interference.
Graphchef \cite{graphchef} produced a set of rules, distinguishing it from existing GNNs and explanation methods that focus on reasoning with individual graphs.
GADAM \cite{chen2024boosting} obtained local anomaly scores conflict-freely and utilized hybrid attention-based adaptive messaging to enable nodes to absorb abnormal or normal signals selectively.

Although these GCN approaches have demonstrated promising performance, they experience diminishing returns or even adverse effects as they progress to deeper layers.

\subsection{Over-smoothing Issue}
The over-smoothing issue refers to the situation when the representation of the nodes in GNNs becomes indistinguishable after multiple aggregation.
Various methods have been proposed to address this problem.
Existing solutions are broadly categorized into three types. 
The first method involves dropping edges \cite{DropEdge, SAEDGEDROP}, while the second imposes regularization constraints on representations \cite{PairNorm, guo2023contranorm}.
 However, both methods primarily mitigate the phenomenon and cannot prevent representations from becoming indistinguishable in deep configuration GCN. 
The third method is adding residuals which is the most widespread and has shown promising results in deep layers.
APPNP \cite{APPNP} achieved this by adding a starting feature matrix at each step of the aggregation process. 
GCNII \cite{GCNII} enhanced APPNP by introducing identity mapping and nonlinearity. 
Ordered GNN \cite{OrderGNN} organized the messaging in node representations by establishing specific blocks of neurons for messaging within a defined number of hops.
NSD \cite{bodnar2022neural} learned sheaf from data and the resulting layer diffusion model had many desirable properties to address the limitations of the classical graph diffusion equation and obtained competitive results in heterogeneous environments.
FLODE  \cite{maskey2023a} introduced a fractional graph Laplace neural ODE, demonstrating that this method effectively propagates information between distant nodes while maintaining a low probability of distant jumps.
GNN-PDE-COV \cite{dan2023re} introduced a selective generalization bias to address the over-smoothing issue in GNNs and developed a novel model to predict in-vivo pathology flow using longitudinal neuroimages.
AGNN \cite{AGNN} mitigated the over-smoothing problem by periodically projecting from low-order feature space to high-order space and utilizing an enhanced Adaboost strategy to aggregate the outputs of each layer.
FROND \cite{kang2024unleashing} employed Caputo fractional derivatives to leverage the non-local properties of fractional calculus, capturing long-term dependencies in feature updates and enhancing graph representation learning, thereby addressing the issue of over-smoothing.

While these methods can construct deep GCNs, they often incorporate components from other fields without addressing the underlying causes of over-smoothing in message passing.

\section{The Proposed Method} \label{SEC3}
As depicted in Figure \ref{Framework}, RF-GCN functions by hierarchical organizing the neighbors of the graph and aggregating them, thus tying the number of network layers directly to the graph data alone, preventing redundant aggregation operations. 
In Section \ref{SECA}, we first introduce the mathematical notation employed. 
Following this, in Section \ref{SECB}, we analyze the causes of over-aggregation and over-smoothing from a mathematical standpoint. 
In Section \ref{SECC}, we present the RF-GCN model. 
Finally, in Section \ref{SECD}, we evaluate the proposed method including the time complexity and over-aggregation analysis.

\subsection{Mathematical Notations} \label{SECA}

We begin by introducing the notation used throughout this paper. Let $\mathcal{G} = \{\mathcal{V}, \mathcal{E}\}$ be a graph, where $\mathcal{V} = \{v_1, \cdots, v_N\}$ represents the set of nodes with $N$ nodes. Denote $\mathbf{A}$ be the adjacency matrix, where $\mathbf{A}_{ij} = 1$ if there is an edge  $(i,j) \in \mathcal{E}$. 
Each node possesses a $d$-dimensional vector $\mathbf{x}_i\in \mathbb{R}^d$, and we define the feature matrix as $\mathbf{X} = [\mathbf{x}_1; \cdots; \mathbf{x}_N]$. 
The symmetrical normalized adjacency matrix of the graph is defined as $ \hat{\mathbf{A}} = \tilde{\mathbf{D}}^{-\frac{1}{2}}\tilde{\mathbf{A}}\tilde{\mathbf{D}}^{-\frac{1}{2}}$, where $\tilde{\mathbf{A}} = \mathbf{A} + \mathbf{I}$ and $\tilde{\mathbf{D}}_{ii} =\sum_j \tilde{\mathbf{A}}_{ij}$.
To enhance the clarity of the paper, we summarize the various symbols and their meanings in Table \ref{SymbolicNormalization}.

\begin{table}[!htbp]
	\centering
	\caption{Symbolic notations with their descriptions.}
	 \label{SymbolicNormalization}
	 \begin{tabular}{l||l}
	  \toprule {Notations} & {Descriptions} \\ \midrule
	  {$\mathbf{X}$}   &  The  feature matrix with $N \times d$.\\
   	{$\mathbf{A}$} &   The adjacency matrix with $N \times N$.\\
	  {$\hat{\mathbf{A}}$} &   The symmetrically normalised adjacency matrix.\\
         {$\mathbf{D}$} &   The degree matrix with $N \times N$.\\
	  {$\mathbf{Q}$}   &  The orthogonal eigenmatrix.\\
	  {$\mathbf{Y}$}   &The label matrix with $N \times c$.\\
   	{$\hat{\mathbf{Y}}$}   & The predictive label matrix with $N \times c$.\\
	  {$L$ }         & The number of layers.\\ 
   	{$\mathbf{H}$ }         & The  graph representation with $N \times d$.\\ 
       {$\mathbf{H}^{(k)}$ }         & The  graph representation in layer $k$.\\ 
        {$\mathbf{C}_l$ }         & The $l$-layer cumulative adjacency matrix.\\ 
   	{$\mathbf{T}_l$ }         & The $l$-layer redundancy-free adjacency matrix.\\ 
   	{$\{\theta_l\}_{l=1}^L$}         & The set of PPR coefficient.\\ 
   	{$\mathbf{\Lambda}$}         & The eigenvalues matrix.\\ 
	  {$\sigma(\cdot)$}         &  The optional activation function.\\ 
	  \bottomrule
	\end{tabular}
   \end{table}
\subsection{Theoretical Analysis} \label{SECB}

We first quantify the extent of neighbor utilization in the GCN aggregation process by defining the average number of neighbor aggregations required to capture global information. This metric is computed as follows:
\begin{equation}
    \mathrm{AvgNAT}(\mathbf{A},L) =\frac{1}{N}\sum_{i=1}^N\sum_{j=1}^N \mathbf{A}^L_{ij} ,
\end{equation}
$\mathrm{AvgNAT}(\mathbf{A}, L)$ measures the average number of neighbor aggregations required to obtain the global information.
\begin{theorem} \label{overaggregate}
For a connected graph $G$ with $N$ nodes, $N-1 \leq \mathrm{AvgNAT}(\mathbf{A}, L) < 2^{(N-2)}$. 
\end{theorem}
\begin{proof}
In a complete graph, the diameter is $1$, so $\mathrm{AvgNAT}(\mathbf{A},1) = N-1$. 
In a chain graph, the diameter is $N-1$, and $\mathrm{AvgNAT}(\mathbf{A}, N-1$)  upper bound of $2^{(N-2)}$. Then we have $ N-1 \leq \mathrm{AvgNAT}(\mathbf{A}, L) < 2^{(N-2)}$. 
\end{proof}
Theorem \ref{overaggregate}  demonstrates that existing message aggregation mechanisms will inevitably over-aggregate low-order information when required to obtain information about the entire graphs. 
Consequently, this may lead to a diminished focus on the nodes' self-information.

We then quantify the extent to which the GCN aggregation process emphasizes its features by defining the node self-attention score as follows:
\begin{equation}
    \mathrm{SAS}(\mathbf{A},K) = \frac{1}{N} \sum_{i=1}^N \frac{\mathbf{A}_{ii}^K}{\sum_{j=1}^N\mathbf{A}_{ij}^K},
\end{equation}
$\mathrm{SAS}(\mathbf{A}, K)$  measures how intensely the GCN focuses on the  self-information 
 of the nodes when accessing $K$-hop neighbor information.
A higher $\mathrm{SAS}(\mathbf{A}, K)$ value indicates increased attention to the node's self-information, whereas a lower $\mathrm{SAS}(\mathbf{A}, K)$  indicates reduced focus.
We then rigorously characterize the over-weakening of the self-attention score resulting from two typical propagation mechanisms.

\begin{theorem} \label{SGCoversmothing}
Considering a connected graph $\mathcal{G} = (\mathbf{A}, \mathbf{X})$ and the message passing process $\hat{\mathbf{A}}^k \mathbf{X}$, where $\hat{\mathbf{A}}=\mathbf{D}^{-\frac{1}{2}}\tilde{\mathbf{A}}\mathbf{D}^{-\frac{1}{2}}$ or  $\hat{\mathbf{A}}=\mathbf{D}^{-1}\tilde{\mathbf{A}}$, we have $\lim_{k \to \infty}\mathrm{SAS}(\hat{\mathbf{A}}, k) = \frac{1}{N}$.
\end{theorem}

\begin{proof}
    Given a symmetric matrix $\hat{\mathbf{A}}=\mathbf{D}^{-\frac{1}{2}}\tilde{\mathbf{A}}\mathbf{D}^{-\frac{1}{2}}$, it can be eigen-decomposed as $\hat{\mathbf{A}} = \mathbf{Q}\mathbf{\Lambda}\mathbf{Q}^T$, where  $\mathbf{Q}$ is an orthogonal eigen-matrix and $\mathbf{\Lambda}$  is a diagonal matrix whose diagonal entries are the eigenvalues of $\hat{\mathbf{A}}$.
    The expression for $\hat{\mathbf{A}}^{k}$  can be derived as follows:
\begin{equation}  
    \begin{aligned}
    \hat{\mathbf{A}}^k =& \mathbf{Q}\mathbf{\Lambda}\mathbf{Q}^T \cdots \mathbf{Q}\mathbf{\Lambda}\mathbf{Q}^T = \mathbf{Q}\mathbf{\Lambda}^{k}\mathbf{Q}^T \\
    =& \begin{bmatrix}
        \mathbf{q}_1, & \mathbf{q}_2 ,&\cdots, & \mathbf{q}_N 
        \end{bmatrix} 
        \begin{bmatrix}
        \lambda_1^k\mathbf{q}_1^T\\  \lambda_2^k \mathbf{q}_2^T\\ \vdots \\   \lambda_n^k\mathbf{q}_N^T 
        \end{bmatrix}
    =  \sum_{i=1}^N\lambda_i^k\mathbf{q}_i\mathbf{q}_i^T.
    \end{aligned}
\end{equation}
Given that the symmetrically normalized adjacency matrix $\hat{\mathbf{A}}$ has one eigenvalue  $\lambda_j = 1$, and all other eigenvalues satisfy the condition $|\lambda_i| < 1 ~(i \neq j)$.
Therefore, its limiting form can be expressed as follows:
\begin{equation}  
    \begin{aligned}
\lim_{k \to \infty}\hat{\mathbf{A}}^k =&\mathbf{q}_j\mathbf{q}_j^T
= \begin{bmatrix}
        q_{j1}^2 & \cdots& q_{j1}q_{jN} \\
        \vdots & \ddots & \vdots  \\ 
        q_{jN}q_{j1}& \cdots& q_{jN}^2
    \end{bmatrix} .
    \end{aligned}
\end{equation}
For any $m$, the following results can be obtained:
\begin{equation}  \label{redistribution}
    \begin{aligned}
 \lim_{k \to \infty}\frac{\hat{\mathbf{A}}_{mm}^k}{\sum_{n=1}^N\hat{\mathbf{A}}_{mn}^k}  = \frac{q_{jm}}{ q_{j1}  + \cdots + q_{jN}}.
    \end{aligned}
\end{equation}
\begin{equation}
\begin{aligned}
        &\lim_{k \to \infty}\mathrm{SAS}(\hat{\mathbf{A}}, k) =  \lim_{k \to \infty} \frac{1}{N} \sum_{i=1}^N \frac{\hat{\mathbf{A}}_{ii}^k}{\sum_{j=1}^N\hat{\mathbf{A}}_{ij}^k}\\
        &=\frac{1}{N}(\frac{q_{j1}}{ q_{j1}  + \cdots + q_{jN}} + \cdots + \frac{q_{jN}}{ q_{j1}  + \cdots + q_{jN}}) = \frac{1}{N}.
\end{aligned}
\end{equation}
When considering another case, i.e., $\hat{\mathbf{A}}=\mathbf{D}^{-1}\tilde{\mathbf{A}}$, where the sum of each row equals 1, $\mathrm{SAS}(\hat{\mathbf{A}}, k)$ can be written as:
\begin{equation}
    \mathrm{SAS}(\hat{\mathbf{A}},k) = \frac{1}{N} \sum_{i=1}^N \hat{\mathbf{A}}_{ii}^k.
\end{equation}
Under this setting, the message aggregation process can be viewed as a Markov chain with a state transfer matrix $\hat{\mathbf{A}}$.
Given that the graph $\mathcal{G}$ is connected and includes self-loops, the Markov chain is both irreducible and aperiodic. 
As $k$ approaches infinity $ \hat{\mathbf{A}}^k$ converges to a steady-state distribution $\mathbf{P}$, where all elements of $\mathbf{P}$ are equal.
 Therefore, the following results can be obtained:
 \begin{equation}
   \lim_{k \to \infty}\hat{\mathbf{A}}^k = \mathbf{P}=
   \begin{bmatrix}
        p_{1} & \cdots& p_{N} \\
        \vdots & \ddots & \vdots  \\ 
         p_{1} & \cdots& p_{N}
    \end{bmatrix},
\end{equation}
where  $ p_{1} + \cdots + p_{N} =1 $. Thus, we have
\begin{equation}
    \lim_{k \to \infty}\mathrm{SAS}(\hat{\mathbf{A}}, k) = \frac{1}{N}(p_1+ \cdots+p_N) =  \frac{1}{N}.
\end{equation}
This completes the proof.
\end{proof}
The theorem reveals that as the model integrates information from higher-order neighbors, there is a gradual neglect of the node's self-information.
Theoretically, this is shown to lead to over-smoothing in deep configuration, where node features across the network become increasingly homogenized and indistinguishable. 
Thus, developing strategies to effectively capture higher-order neighbor information without compromising the distinctiveness of individual node representations is essential in overcoming the over-smoothing challenge.

The existing method of increasing the attention of node self-information by adding initial residuals operates $\beta \hat{\mathbf{A}}+ (1-\beta)\mathbf{I}$, where $\beta$ is a hyperparameter ranging from 0 to 1.
This approach increases the values of the diagonal elements of the adjacency matrix while weakening the values of the non-diagonal elements, thus enhancing the focus on the node's self-information. 
With the initial residual connections, we have
\begin{equation}
   \mathrm{SAS}(\beta \hat{\mathbf{A}}+ (1-\beta)\mathbf{I}, k) > \mathrm{SAS}(\hat{\mathbf{A}},k).
\end{equation}

However, this approach still suffers from the phenomenon of over-aggregate of low-order neighbors because it does not substantially change the operation of message aggregation.

\subsection{Redundancy-Free Graph Convolutional Network} \label{SECC}
Based on the observations above, we can draw the conclusion that the forward propagation in a $k$-step linear graph convolution, regardless of the layer numbers, can be expressed as a linear combination of coefficient redistribution. Namely,
\begin{equation}
    \hat{\mathbf{A}}^k\mathbf{X} =  \mathbf{H}=
    \begin{bmatrix}
        a_{11} & \cdots& a_{1N} \\
        \vdots & \ddots & \vdots  \\ 
        a_{N1}& \cdots& a_{NN}
    \end{bmatrix}
    \cdots
    \begin{bmatrix}
        \mathbf{x}_1  \\ \vdots \\  \mathbf{x}_N
    \end{bmatrix} .
\end{equation}
\begin{equation}
\begin{aligned}
        \mathbf{h}_i =  \alpha_{i1} \mathbf{x}_1 + \cdots + \alpha_{in} \mathbf{x}_N,
\end{aligned}
\end{equation}
where $\{\alpha_{ij}\}_{j=1}^N$ is the set of coefficients after merging.
We group all nodes based on the shortest distance from the node $i$ and rewrite the above equation as
\begin{equation}
     \mathbf{h}_i = \sum_{j_1 \in \mathcal{N}_1}\alpha_{ij_1} \mathbf{x}_{j_1} + \cdots + \sum_{j_L \in \mathcal{N}_L} \alpha_{ij_L} \mathbf{x}_{j_L},
\end{equation}
where $\mathcal{N}_{il}$ represents the set of $l$-order neighbors of the node $i$ and $\mathcal{N}_{i1} \cap \mathcal{N}_{i2}  \cap \cdots \cap  \mathcal{N}_{iL_{i}} = \phi$, where $L_{i}$ represents the maximum value of the shortest distance from the node $i$ to other nodes.
We define the set of cumulative adjacency matrices and redundancy-free adjacency matrices as 
$\{\mathbf{C}_l\}_{l=1}^{L}$ and $\{\mathbf{T}_l\}_{l=1}^{L}$ according to the neighbors set $\{\mathcal{N}_l\}_{l=1}^{L}$, where they are calculated as 
\begin{equation}\label{culmulate}
     \mathbf{C}_l  =f (\mathbf{A}^l + \cdots +\mathbf{A}),
\end{equation}
 where  $f(\cdot)$ converts each non-zero element of the matrix to 1, and 0 otherwise.
$\mathbf{C}_l$ is a matrix that records connections to nodes at $l$-hops and within $l$-hops, with its values are computed as below:
\begin{equation}
     \mathbf{C}_l [i,j] = \left\{
        \begin{aligned}
             & 1 ~~~~~~ dis(i,j)\leq l, \\
             & 0 ~~~~~~ otherwise,
        \end{aligned}
     \right.
\end{equation}
where $dis(i,j)$ denotes the shortest distance between nodes $i$ and $j$.
Then the redundancy-free adjacency matrix $\mathbf{T}_l$ can be calculated through the set of cumulative adjacency matrices:
\begin{equation} \label{pureadj}
    \mathbf{T}_l  = \mathbf{C}_l - \mathbf{C}_{l-1}.
\end{equation}
A symmetric normalization strategy is applied to each redundancy-free matrix:
\begin{equation}
     \hat{\mathbf{T}_l} = \tilde{\mathbf{D}}^{-\frac{1}{2}}_l \tilde{\mathbf{T}_l}\tilde{\mathbf{D}}^{-\frac{1}{2}}_l,
\end{equation}
where $\tilde{\mathbf{T}_l} = \mathbf{T}_l + \mathbf{I}$.
We then fuse these redundancy-free adjacency matrices by applying different coefficients to obtain the final hierarchical sample dependencies, i.e.,
\begin{equation}\label{fused}
\begin{aligned}
\hat{\mathbf{T}} = \sum_{l=1}^L \theta_l \hat{\mathbf{T}}_l,
\end{aligned}
\end{equation}
where the coefficients $\{\theta_l\}_{l=1}^L$ is
\begin{equation}
    \theta_l^{\mathrm{PPR}} = (1-\frac{1}{\alpha})^l.
\end{equation}
Here, the coefficient $\{\theta_l^{\mathrm{PPR}}\}_{l=1}^{L}$ exhibits a gradual decrease, aiming to focus on neighboring nodes within closer hops primarily. 

The final node representations are computed by passing the aggregated features through a Multi-Layer Perceptron (MLP), followed by a softmax layer for classification:
\begin{equation}\label{finalrepresent}
\begin{aligned}
 &\mathbf{Z}= \mathrm{MLP} (\mathbf{X}),\\
&\mathbf{S} = \hat{\mathbf{T}}\mathbf{Z},\\
&\mathbf{H}^L  = \mathrm{MLP} (\mathbf{S}),\\
&\hat{\mathbf{Y}} =\mathrm{Softmax}(\mathbf{H}^L ),
\end{aligned}
\end{equation}
where $\hat{\mathbf{Y}}$ represents the predictive results.
We utilize cross-entropy loss for the loss function, which is expressed as follows:
\begin{equation} \label{loss}
\begin{aligned}
\mathcal{L} = -\sum_{i \in \Omega}\sum_{j = 1}^C \mathbf{Y}_{ij} ln(\hat{\mathbf{Y}}_{ij}),
\end{aligned}
\end{equation}
where $\mathbf{Y}$ is the label matrix,  $\Omega$ is the index set of labeled data and $C$ is the number of classes.

\subsection{Model Analysis} \label{SECD}
We also analyze the proposed model, which generally consists of two parts: the first involves constructing redundant-free adjacency matrices, and the second involves the forward propagation of GCNs.
The  calculation of $\mathbf{T}$ has a computational complexity of $\mathcal{O}(LN^2)$  while the complexity for computing $\hat{\mathbf{Y}}$ is $\mathcal{O}(N^2d+Nd^2)$.
Therefore, the total time complexity is $\mathcal{O}(N^2 (L+d)+Nd^2)$.

In addition, we analyze the average number of neighbor aggregations required for the proposed model to effectively capture global information. The specifics are outlined as follows:

\begin{equation}
\begin{aligned}
        \lim_{k \to \infty}\mathrm{SAS}(\mathbf{T}, k) =& \frac{1}{N}  \lim_{k \to \infty}
\sum_{i=1}^N\frac{\mathbf{T}_{ii}}{\sum_{j=1}^N\mathbf{T}_{ij}}\\
=&\frac{1}{N}  \lim_{k \to \infty}\sum_{i=1}^N \frac{\sum_{l=1}^k \theta_l \hat{\mathbf{T}}_{l,ii}}{\sum_{j=1}^N\sum_{l=1}^k \theta_l \hat{\mathbf{T}}_{l,ij}}\\
>& \frac{1}{N} \times N \times \epsilon = \epsilon,
\end{aligned}
\end{equation}
where $\epsilon = \min\{\frac{\hat{\mathbf{T}}_{l,ii}}{\sum_{j=1}^N \theta_l \hat{\mathbf{T}}_{l,ij}}\}_{l=1}^k \geq \frac{1}{N}$. 
The results indicate that the self-attention score of the proposed method is higher than GCN. Furthermore, we analyze the average neighbors aggregation times of the proposed methods. Namely,
\begin{equation}
    \mathrm{AvgNAT}(\mathbf{T},L) =\frac{1}{N}\sum_{i=1}^N\sum_{j=1}^N \mathbf{T}_{ij} = N-1,
\end{equation}
which demonstrates that the proposed method needs fewer aggregations to obtain the global information of the entire graph.
\begin{algorithm}[t]
  \caption{RF-GCN}
  \textbf{Input}: Graph data $\mathcal{G} = (\mathbf{X}, \mathbf{A})$, label matrix $\mathbf{Y}$, and the hyperparameter $\alpha$.  \\
  \textbf{Output}: Predictive label $\hat{\mathbf{Y}}$.
  \begin{algorithmic}[1]\label{Alg}
  \STATE {Initialize the weight matrix $\mathbf{W}$;}
  \STATE{Initialize the trainable parameters of  $\mathrm{MLP}(\cdot)$; }
    \STATE {Compute the cumulative adjacency
matrices $\{\mathbf{C}_l\}_{l=1}^L$ by Equation \eqref{culmulate};}
  \STATE {Calculate the redundancy-free adjacency matrices $\{\mathbf{T}_l\}_{l=1}^L$ by Equation \eqref{pureadj};}
  \STATE {Compute the fused redundancy-free adjacency matrix  $\mathbf{T}$ by Equation \eqref{fused};}
  \WHILE {not converge}
        \STATE {Calculate $\hat{\mathbf{Y}}$ by Equation \eqref{finalrepresent};}
        \STATE {Compute loss $\mathcal{L}$ by Equation \eqref{loss}}
        \STATE {Update the parameters set  of the $\mathrm{MLP}(\cdot)$ with backward propagation;}
  \ENDWHILE\\
  \RETURN {Predictive label $\hat{\mathbf{Y}}$.}
  \end{algorithmic}
\end{algorithm}

\section{Experiments} \label{SEC4}
In this section, we evaluate the proposed method across various graph datasets, including homogeneous graph node classification, heterophilous graph node classification, and graph classification. For each task, we compare the proposed method with a counterpart specifically designed for that particular task.
More comprehensive experiments including
parameter sensitivity, ablation study, and convergence analysis are conducted to analyze the effectiveness and superiority of the proposed method.

\subsection{Datasets}
We selected sixteen benchmarks for experiments: (1) three citation networks — Cora, Citeseer, and Pubmed;  (2) two social networks — Flickr and BlogCatalog; (3) a paper network - ACM; (4) four heterophily datasets - Cornell, Tesax, Wisconsin and Film.  (5) six graph-level datasets - IMDB-BINARY, IMDB-MULTI, COLLAB, MUTAG, PROTEINS, and PTC.
A brief dataset presentation is shown in Tables \ref{statistics} and \ref{graphstatistics}.

\begin{table}[!htbp]
  \centering
      \caption{A brief description of the node classification datasets.}
    \begin{tabular}{l||cccc}
    \toprule
    Dataset & \# Nodes  & \# Edges  & \# Classes  & \# Features \\
    \midrule
    Cora  & 2,708 & 5,429 & 7     & 1,433 \\
    Citeseer & 3,327 & 4,732 & 6     & 3,703 \\
    Pubmed & 19,717 & 44,338 & 3     & 500 \\
    ACM   & 3,025 & 13,128 & 3     & 1,870 \\
    BlogCatalog & 5,196 & 171,743 & 6     & 8,189 \\
    Flickr & 7,575 & 239,738 & 9     & 12,047 \\
    \midrule
    Cornell & 183 & 295 & 5     & 1,703 \\
    Texas & 183 & 309 & 5     & 1,703 \\
    Wisconsin & 251 & 499 & 5     & 1,703 \\
    Film & 7,600 & 33,544 & 5     & 931 \\
    \bottomrule
    \end{tabular}%
  \label{statistics}%
\end{table}%
\begin{table}[htbp]
  \centering
  \caption{A brief description of the graph classification datasets.}
    \begin{tabular}{l||cccc}
    \toprule
    Dataset & \multicolumn{1}{l}{\# Graphs} & \multicolumn{1}{l}{\# Classes} & \multicolumn{1}{l}{Avg. \# Nodes} & \# Data types \\
    \midrule
    IMDB-B & 1,000 & 2     & 19.8  & Movie \\
    IMDB-M & 1,500 & 3     & 13    & Movie \\
    COLLAB & 5,000 & 3     & 74.5  & Scientific \\
    MUTAG & 188   & 2     & 17.9  & Bioinformatics \\
    PROTEINS & 1,113 & 2     & 39.1  & Bioinformatics \\
    PTC   & 344   & 2     & 25.5  & Bioinformatics \\
    \bottomrule
    \end{tabular}%
  \label{graphstatistics}%
\end{table}%

\begin{itemize}
    \item \textbf{Cora, Citeseer, and Pubmed}\footnote{\href{https://github.com/shchur/gnn-benchmark\#datasets}{https://github.com/shchur/gnn\-benchmark\#datasets}} are three citation networks for research papers, where nodes represent publications and edges denote citation links. Node attributes consist of bag-of-words representations of the papers.
    \item \textbf{ACM}\footnote{\href{https://github.com/Jhy1993/HAN}{https://github.com/Jhy1993/HAN}} is a paper network where nodes represent papers connected by edges if two papers share the same author. The network is characterized by features that include bag-of-words representations of paper keywords.
    \item \textbf{BlogCatalog}\footnote{\href{https://networkrepository.com/soc-BlogCatalog.php}{https://networkrepository.com/soc-BlogCatalog.php}} is a social network where nodes represent users and edges represent their relationships, with labels indicating the topic categories provided by the users.
    \item \textbf{Flickr}\footnote{\href{https://github.com/xhuang31/LANE}{https://github.com/xhuang31/LANE}} is a social network derived from an image and video hosting website. In this network, nodes denote users, edges denote their relationships, and all nodes are categorized into nine classes.
    \item \textbf{Cornell, Texas, and Wisconsin}\footnote{\href{https://github.com/alexfanjn/GeomGCN_PyG}{https://github.com/alexfanjn/GeomGCN\_PyG}\label{Gemodata}} are three sub-datasets. Nodes represent web pages, and edges indicate hyperlinks between them. The node features consist of the bag-of-words representation of web pages. These web pages are manually categorized into five groups: student, project, course, staff, and faculty.
    \item \textbf{Film}\textsuperscript{\ref {Gemodata}} represents the actor-only subgraph of the film-director-actor-writer network.
Each node represents an actor, and the edge between them indicates their co-occurrence on the same Wikipedia page.
     \item  \textbf{IMDB-BINARY} and \textbf{IMDB-MULTI} \footnote{\href{https://github.com/weihua916/powerful-gnns}{https://github.com/weihua916/powerful-gnns}\label{GINDATA}}are datasets where each graph represents an actor's ego-network, with nodes as actors and edges as co-appearances in movies, categorized by movie genre.
    \item  \textbf{COLLAB}\textsuperscript{\ref {GINDATA}} is a scientific collaboration dataset derived from three public sources. Each graph in the dataset represents a researcher's self-network. Similar to movie datasets, these graphs are categorized according to the researchers' domains.
    \item  \textbf{MUTAG}\textsuperscript{\ref {GINDATA}} is a bioinformatics dataset containing 188 mutagenic aromatic and heteroaromatic nitro compounds with 7 labels.
    \item  \textbf{PROTEINS} \textsuperscript{\ref {GINDATA}}  is a bioinformatics dataset where nodes denote secondary structure elements with 3 labels. An edge between any two nodes  indicates that they are neighbors in the amino acid
    sequence or 3D space.
    \item \textbf{PTC}\textsuperscript{\ref {GINDATA}} is a bioinformatics dataset comprising  344 chemical compounds that report carcinogenicity for male and female rats and it includes 19 discrete labels.
\end{itemize}

\begin{table*}[htbp]
  \centering
        \caption{Accuracy and F1-score (mean\% and standard deviation\%) of all methods, where the best results are marked in bold and the second-best are in underlined.}
    \resizebox{\textwidth}{!}{
    \begin{tabular}{m{0.8cm}<{\centering}|m{2.2cm}|m{1.6cm}<{\centering}m{1.6cm}<{\centering}m{1.6cm}<{\centering}m{1.6cm}<{\centering}m{1.6cm}<{\centering}m{1.6cm}<{\centering}}
    \toprule
    \multicolumn{1}{c|}{Metric} & Methods / Datasets & Cora & Citeseer &Pubmed& ACM & BlogCatalog & Flickr\\
    \midrule
   
    \multirow{10}[2]{*}{ACC} & GCN   & 80.92 (0.20) & 68.12 (0.72) & 77.78 (0.24)& 85.80 (0.97) & 73.06 (1.00)  & 51.24 (0.95)\\
    &JK-Net & 79.66 (0.47) & 68.74 (0.64) & 73.40 (0.95)& 80.32 (0.36) & 79.56 (0.29) &  49.80 (0.00)\\
    &APPNP & 81.60 (1.05) & 67.58 (1.15) &79.22 (1.57)& 87.64 (0.94) & 81.44 (2.16) &  52.40 (2.04)\\
    &SGC   & 81.06 (0.38) & 70.26 (0.08) &79.82 (1.39)& 85.32 (0.77) & 82.04 (0.29)  & 55.04 (0.34)\\
    &SSGC  & 81.56 (0.27) & 67.36 (0.71) & 80.26 (0.05)& 85.13 (2.88) & 81.20 (0.26) & 55.54 (0.33)\\

    &AdaGCN & 81.00 (0.02) & 70.10 (0.00) & 76.30 (0.20)& 84.70 (5.66) & 83.60 (1.16) & 59.20 (1.10) \\
    &DAGCN & 81.54 (0.43) & 70.14 (1.15) &80.02 (0.31)& 87.36 (0.85) & \underline{87.94 (0.61)} & \underline{61.42 (0.52)}\\
    &GCNII  & \underline{82.12 (0.41)} & \underline{70.40 (0.81)}& \underline{80.26 (0.32)} &\underline{90.84 (0.53)}  & 87.76 (2.65)  &56.48 (1.56)\\
    &DGC-Euler  & 80.49 (0.08) & 69.78 (0.06)& 80.22 (0.06) &85.09 (0.15)  & 83.72 (0.08)  &58.84 (0.05)\\
\cmidrule{2-8}    &RF-GCN & \textbf{83.72 (0.50)} & \textbf{71.88 (0.72)} & \textbf{80.92 (0.61)} &\textbf{91.40 (0.74)} & \textbf{88.08 (1.31)} &   \textbf{64.00  (2.33)}\\
    \midrule
    \multirow{10}[2]{*}{F1} & GCN   & 79.60 (0.29) & 64.03 (0.39) & 77.50 (0.21)& 85.87 (0.95) & 71.43 (1.20)  & 50.30 (0.58)\\
    &JK-Net & 79.66 (0.47) & 65.44 (0.77) & 72.17 (0.88)& 80.23 (0.38) & 79.56 (0.29) &  49.80 (0.00)\\
    &APPNP & 79.35 (0.91) & 65.45 (1.05) &78.88 (1.30)& 87.64 (0.94) & 79.48 (3.14) &  53.00 (2.30)\\
    &SGC   & 79.90 (0.32) & 65.00 (0.05) &79.76 (1.39)& 85.48 (0.41) & 82.04 (0.29)  & 53.13 (0.40)\\
    &SSGC  & 80.49 (0.20) & 63.25 (0.57) & 79.63 (0.42)& 85.12 (2.99) & 81.23 (0.23) & 57.50 (0.31)\\

    &AdaGCN & 79.16 (0.27) & 67.16 (0.00) & 76.23 (0.25)& 59.88 (5.66) & 76.06 (3.57) & 58.50 (1.23) \\
    &DAGCN & 79.64 (0.43) & 64.89 (0.62) &79.84 (0.59)& 87.79 (0.35) & \textbf{87.79 (0.58)} & \underline{62.44 (0.68)}\\
    &GCNII  & \underline{80.78 (0.78)} & \underline{68.20 (0.62)}& 79.77 (0.48) &\underline{91.08 (0.51)}  & 85.69 (1.65)  &59.54 (1.53)\\
    &DGC-Euler  & 79.71 (0.08) & 64.11 (0.05)& \textbf{80.38 (0.06)} & 85.25 (0.15)  & 83.45 (0.09)  &58.94 (0.04)\\
    \cmidrule{2-8}          &RF-GCN & \textbf{82.23 (0.69)} & \textbf{68.71 (0.76)} & \underline{80.14 (0.56)} &\textbf{91.53 (0.70)} & \underline{87.39 (1.36)} &   \textbf{63.88 (1.90)}\\
    \bottomrule
    \end{tabular}}
  \label{semiclassify}%
\end{table*}%
\subsection{Compared Methods}
The methods employed for homophily and heterophily graph datasets are compared here.
GCN \cite{GCN}, SGC \cite{SGC}, SSGC \cite{SSGC}, APPNP \cite{APPNP}, JKNet \cite{JKNET}, AdaGCN \cite{AdaGCN}, DAGCN \cite{DAGCN}, GCNII \cite{GCNII} and DGC-Euler \cite{DGC} for homophily  graph datasets.
Geom-GCN \cite{Geom-GCN:}, HOG-GCN \cite{HOG-GCN}, LINKX \cite{LINKX}, ACM-GCN \cite{ACM-GCN} and LRGNN \cite{LRGNN} for heterophily graph datasets.
 DCNN \cite{DCNN}, PATCHY-SAN \cite{PATCHY-SAN}, DGCNN \cite{DGCNN},  GIN \cite{GIN}, GraphSAGE \cite{GRAPHSAGE}, and GCKM \cite{wuzhihaonips} for graph level datasets. The details of these compared methods are shown below.
\begin{itemize}
    \item \textbf{GCN}  learns node representations by aggregating information from neighbors using a first-order approximation of spectral convolutions on graphs.
    \item  \textbf{GraphSAGE}  is an inductive learning framework that efficiently generates unknown vertex embedding using vertex attribute information.
    \item \textbf{SGC} eliminates nonlinearities between GCN layers, condensing the resulting function into a single linear transformation, aiming to reduce unnecessary complexity. 
    \item \textbf{SSGC} introduces a variant of GCN through the adoption of a modified Markov diffusion kernel, enabling exploration of both global and local contexts of nodes. 
    \item \textbf{APPNP} introduces personalized PageRank-based propagation of neural predictions, requiring fewer parameters and less training time compared to alternative methods.
    \item \textbf{JKNet} flexibly adjusts the range of neighborhood information for each node, allowing for improved structure-aware representations.
    \item \textbf{AdaGCN} integrates learned knowledge from different layers of GCN in an Adaboost-like manner, updating layer weights iteratively.
    \item \textbf{DAGCN} constructs deeper graph neural networks that involve segregating transformation and propagation into two distinct operations.
    \item \textbf{GCNII} resolves the issue of over-smoothing by introducing residuals to GCNs with identity maps.
    \item \textbf{DGC-Euler} mitigates over-smoothing by numerically approximating the heat diffusion equation on the graph and constraining the propagation step size to match the total propagation time.
    \item \textbf{Geom-GCN} introduces an aggregation method for graph neural networks that enhances their ability to understand the node structure in neighborhoods and capture long-range dependencies in non-assortative graphs.
    \item \textbf{HOG-GCN} modifies its propagation and aggregation processes to accommodate the homogeneity or heterogeneity of node pairs, moving beyond the typical assumption of uniform node characteristics.
    \item \textbf{LINKX} is trained and evaluated using a mini-batch approach, which helps avoid performance degradation.
    \item \textbf{ACM-GCN} can adaptively use aggregation, diversification, and identity channels for each node, enabling more comprehensive extraction of localized information suited to various node heterogeneity scenarios.
    \item \textbf{LRGNN} predicts the label relationship matrix by addressing a robust low-rank matrix approximation problem, analyzing the low-rank global label relationship matrix specific to signed graphs.
    \item \textbf{DCNN} employs diffusion convolution operations to extract graph-structured data representations, enabling efficient polynomial-time predictions.
    \item \textbf{PATCHYSAN}  is a method that extracts locally connected regions from graphs, similar to how convolutional networks in image processing target locally connected areas in their inputs.
    \item \textbf{DGCNN} preserves detailed vertex information and learns from the overall graph topology, using spatial graph convolution to process unordered vertex features as input.
    \item \textbf{GIN}  is recognized as one of the most expressive GNN architectures, matching the capabilities of the Weisfeiler-Lehman graph isomorphism test.
    \item \textbf{GCKM}  integrates implicit feature mapping induced by kernels with neighbor aggregation over graphs, offering a novel paradigm for graph-based machine learning.
\end{itemize}

\subsection{Parameter Setups}
For all experiments, we randomly divide samples into a training set consisting of $20$ labeled samples per class, a validation set of $500$ samples, and a test set of $1,000$ samples for the homophily graph datasets. 
For the heterophily graph datasets using the identical train/validation/test splits of 60\%, 20\%, and 20\%, respectively.
For graph-level classification, we implemented 10-fold cross-validation following the GIN setup.
The configurations of the compared baseline methods are set to their default values as specified in their original papers.
For RF-GCN, we utilize the Adam optimizer and set the learning rate and weight decay to $1e^{-2}$ and $5e^{-3}$ respectively. The number of neurons in the hidden layer is 64, the dropout rate is set to 0.5, and the hyperparameter $\alpha$ is chosen from $\{1, 2, 5, 10\}$ for different datasets.
All experiments are run on a platform with AMD R9-5900X CPU, NVIDIA GeForce RTX  4080 16G GPU and 32G RAM.
\begin{table*}[h!]
\setlength{\tabcolsep}{6.5pt}
  \centering
  \caption{Graph classification accuracy (mean\% and standard deviation\%) of all methods, note that
the best results are highlighted in bold and the second-best results highlighted are in underlined.}
  \label{graphclass}
  \resizebox{0.9\textwidth}{!}{
  \begin{tabular}{m{2.3cm}|m{1.45cm}<{\centering}m{1.45cm}<{\centering}m{1.45cm}<{\centering}m{1.45cm}<{\centering}m{1.45cm}<{\centering}m{1.45cm}<{\centering}}
    \toprule
    Methods / Datasets& IMDB-B & IMDB-M & COLLAB & MUTAG & PROTEINS & PTC \\
    \midrule
    PATCHYSAN & 71.0 (2.2) & 45.2 (2.8) & 72.6 (4.2) & \textbf{92.6 (4.2)} & 75.9 (2.8) & 60.0 (4.8) \\
    DGCNN & 70 & 47.8 & 73.7 & 85.8 & 75.5 & 58.6 \\
    AWL & 74.5 (5.9) & 51.5 (3.6) & 73.9 (1.9) & 87.9 (9.8) & - & - \\
    MLP & 73.7 (3.7) & 52.3 (3.1) & 79.2 (2.4) & 84.0 (6.1) & 76.0 (3.2) & 66.6 (6.9) \\
    GIN & 75.1 (5.1) & 52.3 (2.8) & 80.2 (1.9) & 89.4 (5.6) & \underline{76.2 (3.2)} & 64.6 (7.0) \\
    GCN & 74.0 (3.4) & 51.9 (3.8) & 79.0 (1.8) & 85.6 (5.8) & 76.0 (3.2) & 64.2 (4.3) \\
    GraphSAGE & 72.3 (5.3) & 50.9 (2.2) & - & - & 75.9 (3.2) & 63.9 (7.7) \\
    GCKM & \underline{75.4 (2.4)} & \underline{53.9 (2.8)} & \textbf{81.7 (1.5)} & 88.7 (7.6) & 74.5 (3.9) & \underline{67.7 (5.4)} \\
        \midrule
    RF-GCN & \textbf{78.9 (2.4)} &  \textbf{54.9 (2.9)} & \underline{81.4 (1.5)} & \underline{91.5 (4.9)} & \textbf{78.6 (2.4)} & \textbf{68.6 (3.4)} \\
    \bottomrule
  \end{tabular}}
\end{table*}

\begin{figure*}[!htbp]
	\centering
	\includegraphics[width=0.97\textwidth]{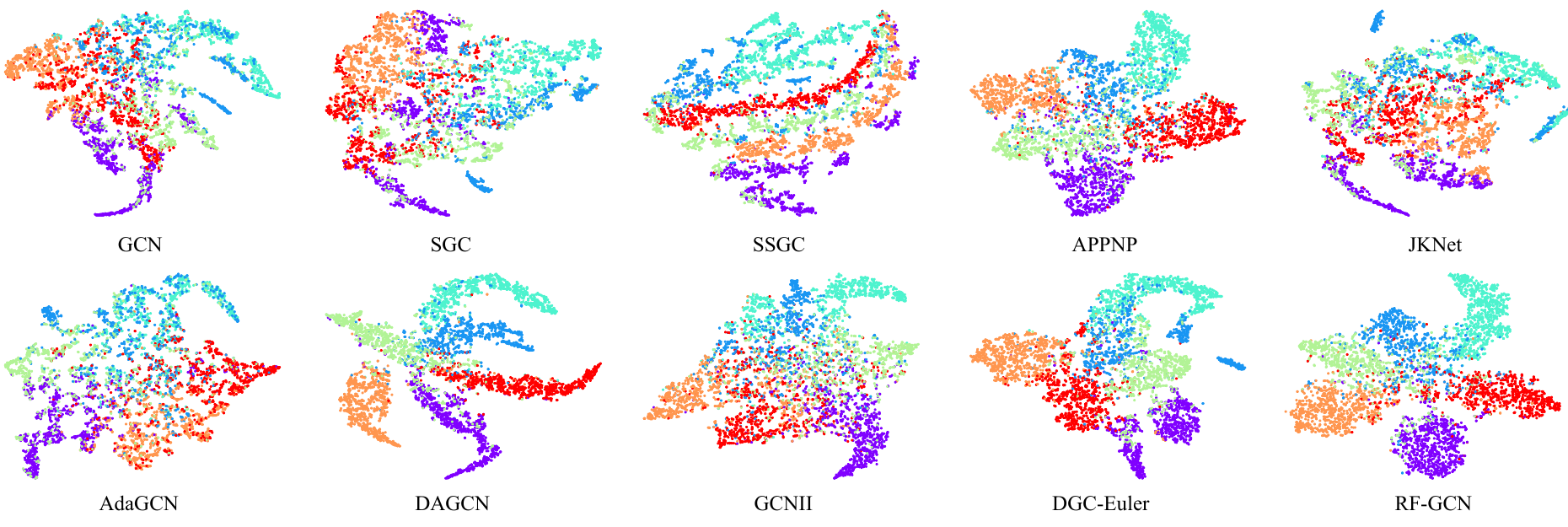}\\
	\caption{The visualization of the representations of all compared methods on the BlogCatalog dataset.}
	\label{Tsne}
\end{figure*}

\begin{table*}[htbp]
  \centering
        \caption{Results with different training set sizes on the Cora dataset in terms of classification accuracy (\%).}
    \resizebox{0.9\textwidth}{!}{
    \begin{tabular}{m{2.8cm}|m{0.7cm}<{\centering}m{0.7cm}<{\centering}m{0.7cm}<{\centering}m{0.7cm}<{\centering}m{0.7cm}<{\centering}m{0.7cm}<{\centering}m{0.7cm}<{\centering}m{0.7cm}<{\centering}m{0.7cm}<{\centering}m{0.7cm}<{\centering}m{0.7cm}<{\centering}}
    \toprule
    Methods / Training size & 1     & 2     & 3     & 4     & 5     & 10    & 20    & 30    & 40    & 50 &100 \\
    \midrule
    GCN   & 32.90 & 34.20 &  49.80 & 56.50  & 59.60 & 66.70  & 79.20  & 79.00 & 79.60 &81.20&85.70\\
    SGC   & 38.40  & 37.80 & 58.80 & 62.10 & 62.71&  68.70  & 76.80 & 78.80  & 79.60 & 79.70 &85.20\\
    APPNP & 37.90  &  63.60 &  63.50 & 71.20 & 77.00  & 76.10 & 79.90 & 80.50 & 81.70 &81.90& 84.50\\
    JKNet & 61.90  & 53.60 & 61.30  & 56.50 & 61.00 & 68.20 & 77.80 & 78.80 & 84.50 & 84.20 &85.20\\
    SSGC  & 37.10 & 42.80 & 60.30 & 65.80 & 67.70 & 72.40 & 78.20& 78.40 & 80.30 & 80.30 & 85.00\\
    AdaGCN & 43.10  &  52.80  & 53.40 & 66.40 &  70.80 & 74.20 & 78.90 & 80.00  &  81.80  &\textbf{85.10}&83.50\\
    DAGCN & 49.70 & 55.10  & 70.50 &  69.10  & 69.80 &  76.30 & 81.20 & 81.30 & 82.30 & 82.80 &\textbf{87.50}\\
    GCNII & 59.80 & 52.50  & 64.80  & 71.20 &74.40& 73.80  & 83.20  & 83.80 &  84.60 &84.60  &87.00\\
    DGC-Euler & 38.72 & 42.23  & 62.82  & 67.42 &71.25& 74.60  & 80.54  & 80.32 &  80.96 &81.02  &84.80\\
        \midrule
    RF-GCN & \textbf{67.00}  & \textbf{63.60} & \textbf{73.00}  &\textbf{77.20} & \textbf{79.40} &\textbf{77.80} & \textbf{84.80}  &  \textbf{84.00} &  \textbf{84.80}   & 84.20 & 86.80 \\
    \bottomrule
    \end{tabular}}
  \label{trainsize}%
\end{table*}%

\begin{table}[!tbp]
  \centering
      \caption{Node classification accuracy (\%) on heterophily  datasets.}
    \resizebox{\linewidth}{!}{
    \begin{tabular}
    {m{1.4cm}|m{1.2cm}<{\centering}m{0.9cm}<{\centering}m{1.2cm}<{\centering}m{0.9cm}<{\centering}}
    \toprule
    Methods & Cornell & Tesax & Wisconsin & Film \\
    Homophily  & 0.3   & 0.11  & 0.21  & 0.22 \\
    \midrule
    Geom-GCN &   47.42   &   74.85    &  64.80  & 26.88 \\
    HOG-GCN & 71.43  &   68.83   &  75.27 &  32.76\\
    LINKX &  77.84  & 74.60 &75.40 & -  \\
    ACM-GCN &   85.71   &  84.00  & 87.20& 35.67 \\
    LRGNN &  86.48  & 90.27 &  \textbf{88.23} & 24.37 \\
        \midrule
    RF-GCN & \textbf{94.29} & \textbf{94.29}    & 88.00  & \textbf{36.18} \\
    \bottomrule
    \end{tabular}}%
  \label{Homophily}%
\end{table}%

\subsection{Semi-supervised Classification}
\textbf{Node Classification on Homophily  Graph}: 
We conduct a semi-supervised classification of nodes on six datasets, and the results for homophily graph datasets are presented in Table \ref{semiclassify}.
It is observed that RF-GCN achieves encouraging results on the majority of the datasets, surpassing the second-highest algorithm by $1.60\%$, $1.48\%$, and $0.70\%$ on the datasets Cora, Citeseer, and Pubmed, respectively.

\textbf{Node Classification on Heterophily  Graph:} The results for heterophily graph datasets are displayed in Table \ref{Homophily}. The table illustrates that RF-GCN consistently achieves promising results when compared with other algorithms. 
Specifically, the proposed method outperformed the suboptimal algorithm by $0.38\%$, $4.02\%$, and $0.51\%$ on the datasets Cornell, Texas, and Film, respectively.
Subpar results are observed solely on the Wisconsin dataset.

\begin{figure*}[!htbp]
	\centering
	\includegraphics[width=\textwidth]{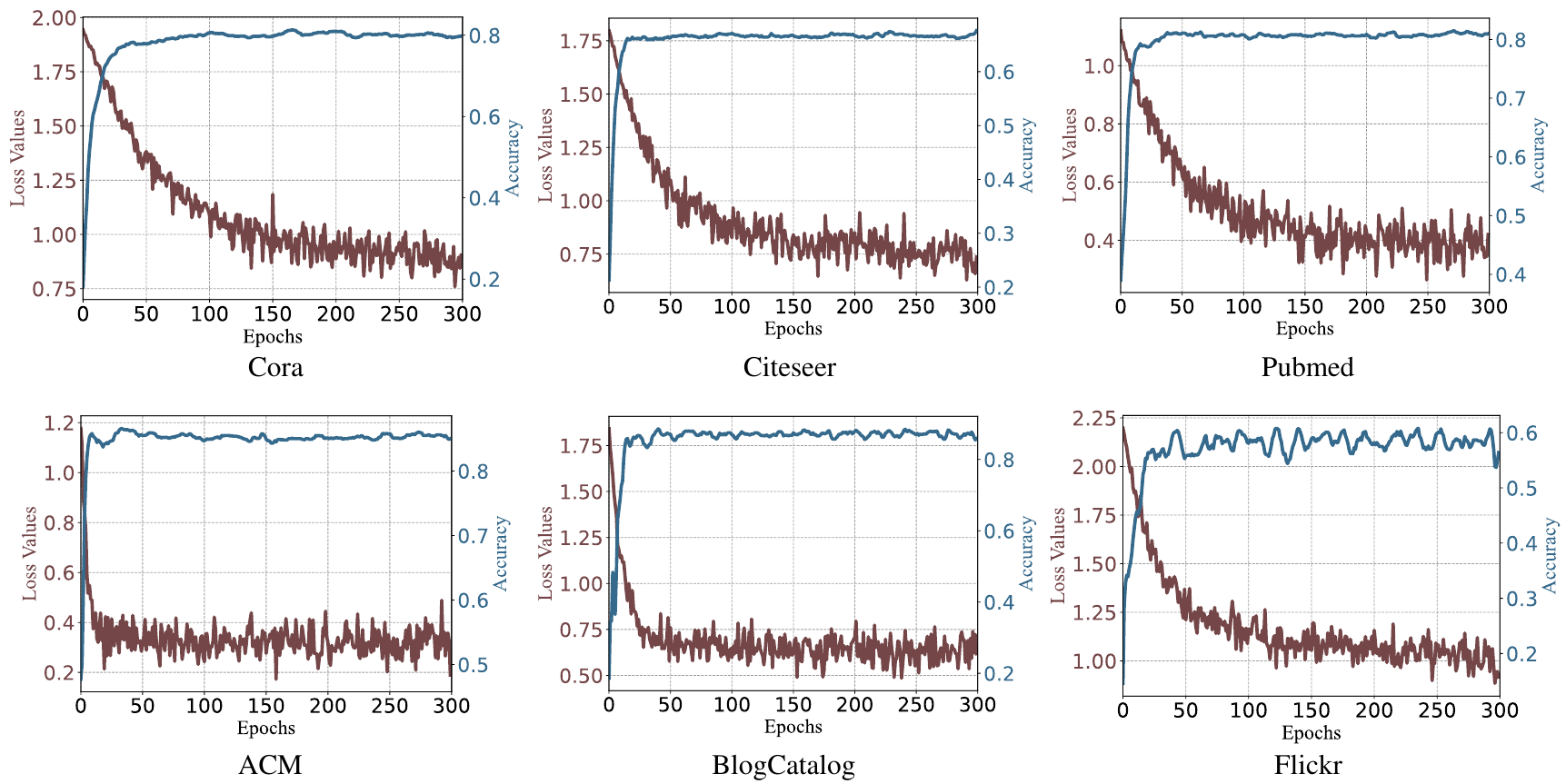}\\
	\caption{The convergence curves of training loss values and validate accuracy with RF-GCN on six datasets.}
	\label{Convergence}
\end{figure*}
\begin{figure*}[!htbp]
	\centering
	\includegraphics[width=0.97\linewidth]{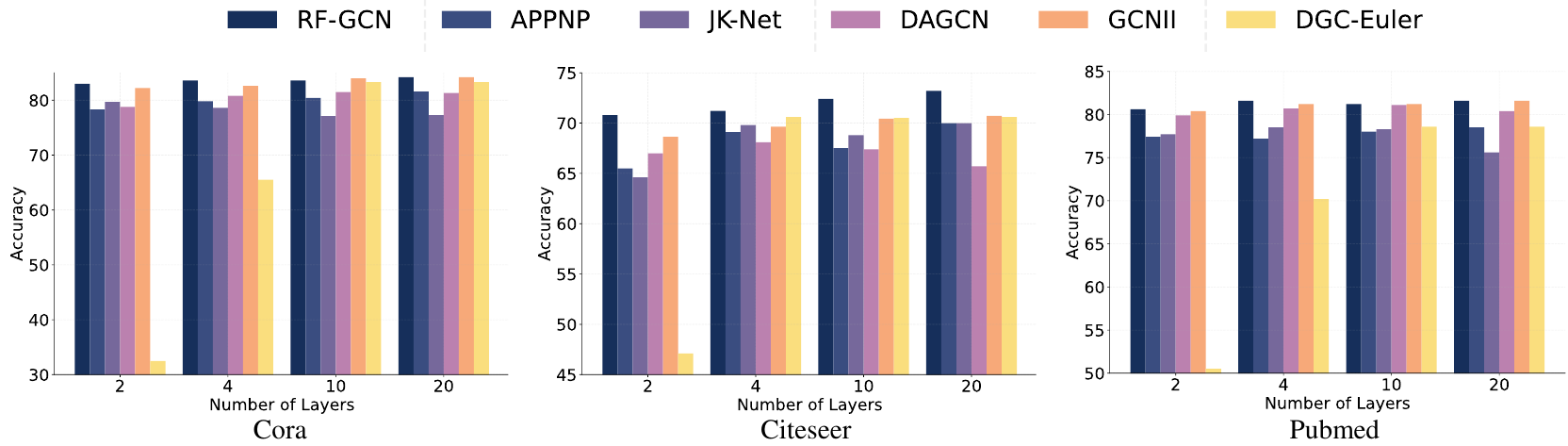}\\
	\caption{Classification accuracy results (\%) for different layers on the dataset Cora, Citeseer, and Pubmed.}
	\label{Depth}
\end{figure*}

\textbf{Graph Classification:}
The classification results for graph-level datasets are presented in Table \ref{graphclass}. The proposed algorithm demonstrates optimal performance on most datasets, outperforming the second-ranked models on IMDB-B, IMDB-M, PROTEINS, and PTC by margins of 3.5\%, 1.0\%, 2.4\%, and 0.9\%, respectively.
Only in the datasets COLLAB and MUTAG does the proposed algorithm underperform relative to GCKM and PATCHYSAN, with accuracies lower by 0.3\% and 1.1\%, respectively.

\textbf{Visualization}:
To better demonstrate the performance of the proposed method,
we take a closer step by performing t-SNE dimensionality reduction on the representations obtained by each algorithm and color the samples according to the ground truth. 
The results are illustrated in Figure \ref{Tsne}. 
From the figure, it is evident that the proposed method exhibits better intra-class aggregation as well as inter-class distinguishability.
This improvement can be attributed to the fact that RF-GCN minimizes the repetitive aggregation of nodes, making different nodes more distinguishable. These observations confirm the superiority of RF-GCN in node representation learning.

\begin{figure*}[!htbp]
	\centering
	\includegraphics[width=0.97\linewidth]{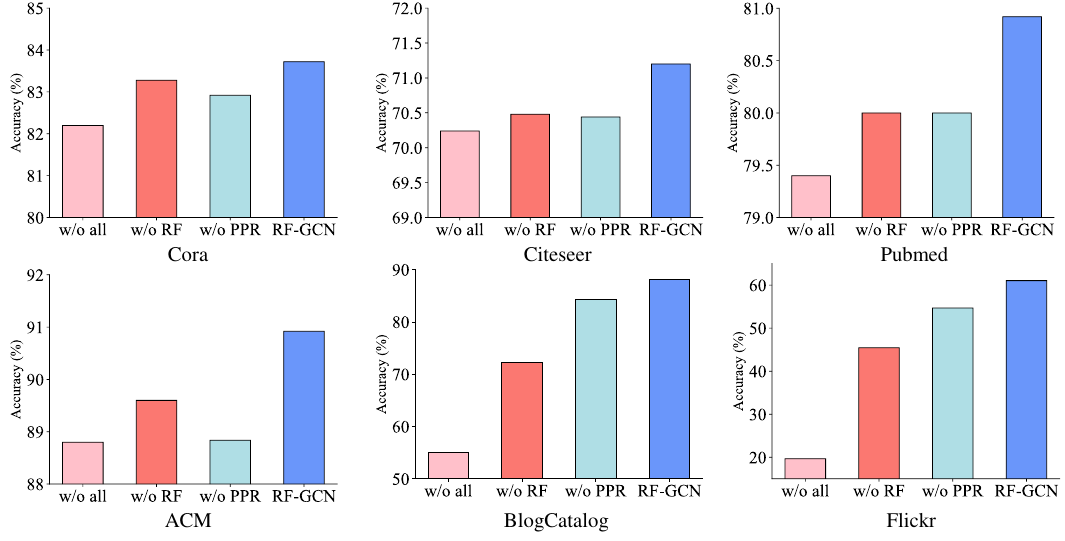}\\
	\caption{The classification accuracy of RF-GCN w.r.t hyperparameter $L$ and $\alpha$ on six datasets.}
	\label{ABLANTION}
\end{figure*}
\begin{figure*}[!htbp]
	\centering
	\includegraphics[width=\linewidth]{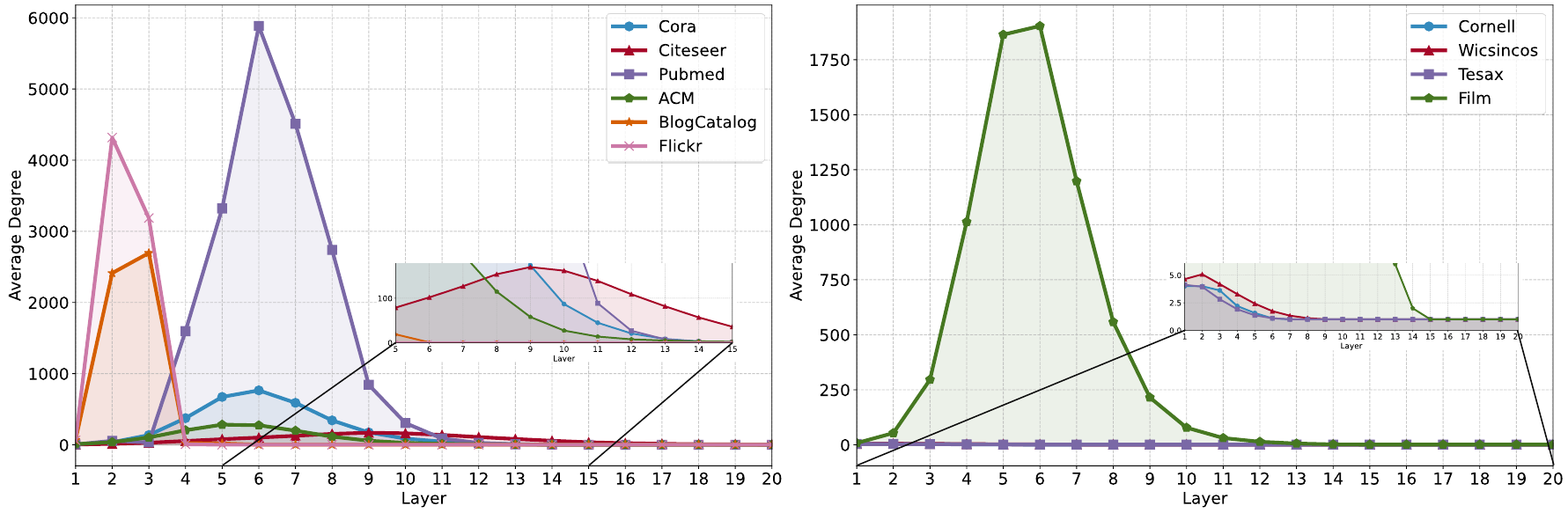}\\
	\caption{The average degree of the redundancy-free adjacency matrices in various layers on six tested datasets.}
	\label{Degree}
\end{figure*}

\textbf{Train Set Size}: 
Given that the number of nodes used for training significantly influences the applicability of the model in real-life situations, we conducted experiments with different training set sizes.
The results shown in Table \ref{trainsize} indicate that even with a minimal training set, RF-GCN outperforms all state-of-the-art methods by a significant margin.
This suggests that RF-GCN is particularly well-suited for real-life scenarios where training samples are limited.
This can be attributed to the capability of RF-GCN to mitigate the impact of indistinguishable parts in the deep layers, providing a heightened sense of versatility compared to GCNs.

\textbf{Convergence Analysis:}
Figure \ref{Convergence} illustrates the convergence curve of the proposed algorithm. As shown, with increasing elapsed time, the training set loss consistently decreases, and the validation set accuracy progressively improves.  Notably, the accuracy on the ACM and BlogCatalog datasets reaches high and stable values more rapidly, at approximately 150 epochs, while other datasets show slower convergence, around 300 epochs.

\begin{figure*}[!htbp]
	\centering
	\includegraphics[width=0.97\textwidth]{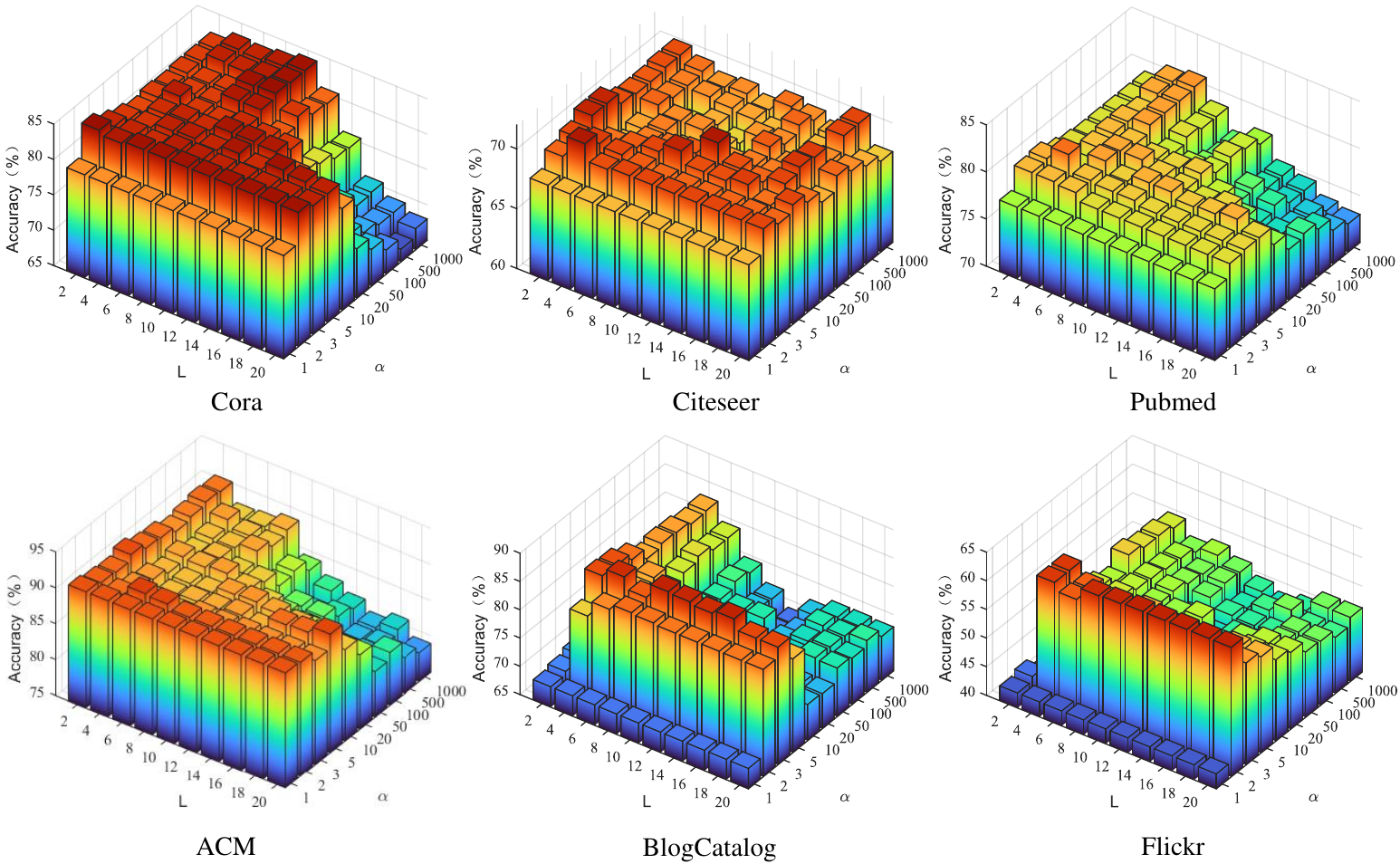}\\
	\caption{The classification accuracy of RF-GCN w.r.t hyperparameter $L$ and $\alpha$ on six datasets.}
	\label{sensitive}
\end{figure*}


\textbf{Over-smoothing Test:}
To verify the performance of the proposed method across multiple layers, we compared the efficacy of various approaches at different depths, with results presented in Figure \ref{Depth}.
As shown in the table, with increasing layers, the accuracy of RF-GCN improves across datasets: from 83.0\% to 84.3\% in the dataset Cora, from 70.8\% to 73.2\% in the dataset Citeseer, and from 80.6\% to 81.6\% in the dataset Pubmed.
This indicates that the network structure can effectively utilize the additional layers to enhance classification accuracy. 
Across all layer configurations, RF-GCN demonstrates the highest accuracy, highlighting its robust data-fitting and generalization capabilities.

\subsection{Ablation Study}
In this section, we validate the effectiveness of the proposed method by systematically removing modules, both with and without redundancy-free adjacency matrices and PPR coefficients. The results are presented in Figure \ref{ABLANTION}. 
The table indicates a gradual performance improvement as the modules are sequentially stacked. 
It is noteworthy that the proposed method shows significantly degraded performance on the datasets BlogCatalog and Flickr when redundancy-free adjacency matrices and PPR coefficients are not utilized.
This performance degradation may be due to the denser node connectivity and higher heterophilous rates of these two datasets. Such conditions often decrease performance when multiple layers are stacked in the proposed model.


\subsection{Parameter Sensitivity Analysis}
Figure \ref{Degree} illustrates the average degree of redundancy-free matrices across ten datasets with varying numbers of layers.
As observed from the figure, the average degree exhibits an initial increase followed by a decrease. For datasets such as Cora, Citeseer, Pubmed, Film, and ACM, the trend is smoother, reaching its peak at layers $6$ and $7$ before gradually converging to $0$ at layer $15$.
In contrast, datasets like Blogcatalog, Cornell, Wicsincos, Tesax, and Flickr display a more undulating pattern, reaching the highest point at layers $2$ and $3$, before eventually converging to $0$ at layers $6$.
This phenomenon can be attributed to the low number of graph nodes and the high number of edges in both datasets, leading to a dense graph. As a result, a very small number of hops can effectively cover the entire graph.
In summary, for most graph datasets, a minimal number of hops is sufficient to capture global information.

Figure \ref{sensitive} illustrates the variation in accuracy for different numbers of layers and various hyperparameters $\alpha$. 
For most datasets, the performance can be maintained at a high level as the number of layers increases when the value of $\alpha$ is within the range $[2,5]$.
One possible reason for this is that when $\alpha$ is too large, high-order neighbors are given the same importance as low-order neighbors, leading to a decrease in performance.

\section{Conclusions} \label{SEC5}

In this paper, we analyzed how existing GCNs over-utilize the information from low-order neighbors when acquiring information over long distances, thus resulting in diminished attention to a node's self-information. 
Although adding residuals can improve attention to the node's self-information, the phenomenon of over-utilizing low-order neighbors persists due to the unchanged message aggregation mechanism.
Therefore, we introduced a framework called RF-GCN. 
This network achieved its objective by hierarchical organizing the node neighbors, with the number of layers being solely related to the graph structure itself. Experiments on sixteen real-world datasets for a node classification task demonstrate that RF-GCN outperforms several state-of-the-art competitors.

This study identifies several promising avenues for further research. While current efforts have largely centered on addressing the over-smoothing problem, it is evident that in many real-world applications, only shallow GCNs are necessary to achieve promising results, with a little additional benefit gained from deeper networks. This raises the question of the necessity and potential benefits of building deeper GCNs, which warrants further investigation.

\bibliographystyle{ieeetr}
\bibliography{tnnls}
\end{document}